\documentclass[11pt]{amsart}
\usepackage[cm]{fullpage}
\usepackage{amssymb,amsmath,amsthm,amscd,mathrsfs,graphicx}
\usepackage[cmtip,all]{xy}
\usepackage{pb-diagram}
\usepackage{tikz}

\numberwithin{equation}{section}
\newtheorem{teo}{Theorem}[section]

\newtheorem{lem}[teo]{Lemma}

\newtheorem{theorem}{Theorem}[section]

\newtheorem{lemma}{Lemma}[section]
\newtheorem{corollary}{Corollary}[section]

\newcommand{\ov}[1]{\overline{#1}}

\newcommand{\ve}{\varepsilon}

\newcommand{\RP}{\mathbb{RP}}

\theoremstyle{definition}

\newcommand{\wmin}{w_{\min}}

\theoremstyle{remark}

\begin{document}
\bibliographystyle{amsplain}

\title{The minimal width of neural networks \\ and immersions of real projective space}

\author{BaoZhi Chu and Ben Weinkove}

\address{Department of Mathematics, Northwestern University, 2033 Sheridan Road, Evanston, IL 60208, USA}

\begin{abstract}
We investigate the minimal width of deep neural networks that can approximate all continuous functions from $\mathbb{R}^n$ to $\mathbb{R}^m$  in the topology of uniform convergence on compact subsets.  We show that whenever the real projective space $\mathbb{RP}^{n-1}$ admits a smooth immersion into $\mathbb{R}^n$, the minimal width is at least $n+\min\{n,m\}$, assuming that the activation function is continuous and monotone.  As a consequence, the minimal width of ReLU networks is $n+m$ whenever $n$ is equal to $2, 3, 4, 7$ or $8$ and $1\le m \le n$, extending results of Hanin-Sellke and Hwang.  
\end{abstract}

\maketitle

\vspace{-20pt}

\section{Introduction}

A fundamental problem concerning expressivity of neural networks is to understand what architectures can approximate arbitrary continuous functions.  

We define a neural network 
$F:\mathbb R^n \to\mathbb R^m$ to be a map of the form
$$
F=A^{(L)}\circ\sigma\circ A^{(L-1)}\circ\cdots
\circ\sigma\circ A^{(1)},
$$
where $A^{(\ell)}: \mathbb{R}^{d_{\ell-1}} \rightarrow \mathbb{R}^{d_{\ell}}$ are affine maps, $L$ is a positive integer called the depth and $$n=d_0, d_1,  \ldots, d_{L-1}, d_L=m$$ are the dimensions of $F$.  The map $\sigma:\mathbb{R} \rightarrow \mathbb{R}$ is a continuous activation function, which is extended to arbitrary dimensions $\sigma: \mathbb{R}^k \rightarrow \mathbb{R}^k$ by $\sigma(x_1, \ldots, x_k) = (\sigma(x_1), \ldots, \sigma(x_k))$.  The architecture of $F$ consists of the dimensions $d_0, \ldots, d_L$ and the choice of activation $\sigma$.  The width of $F$ is 
$$
\operatorname{width}(F):=\max_{0\leq \ell \leq L}d_{\ell}.
$$

A now classical result is the universal approximation theorem for shallow networks.  This states that  networks of depth $2$ (shallow networks) and arbitrarily large width can approximate any continuous function $f: \mathbb{R}^n \rightarrow \mathbb{R}^m$ uniformly on compact subsets, as long as the activation function is not a polynomial \cite{Cybenko, Funahashi, HN, HSW, Leshno}.  

On the other hand, narrow deep networks, which have bounded width but arbitrarily large depth,  have powerful approximation properties (see \cite{Telgarsky, Yarotsky2017}, for example).  A natural question that arises is the  \emph{minimal width problem}. Given an activation $\sigma$, we ask for the least integer $w$ such that any continuous function  $f: \mathbb{R}^n \rightarrow \mathbb{R}^m$ can be approximated uniformly on compact subsets by networks of arbitrary depth and width at most $w$.  Surprisingly, for most values of $n$ and $m$, and standard activations, this minimal width, which we denote by $w_{\text{min}}(n,m)$, remains unknown. Here and throughout the paper, we set $\wmin(n,m)=\infty$ if the minimal width does not exist.

For the ReLU activation, 
$$\sigma(x) = \max\{ x,0\},$$
Hanin--Sellke \cite{HS} showed that
\begin{equation} \label{HS1}
w_{\text{min}}(n,m) \le n+m,
\end{equation}
and equality holds when the output dimension $m$ is $1$:
$$w_{\text{min}}(n,1) = n+1.$$
When the output dimension is strictly larger than 1, the space of maps $f: \mathbb{R}^n \rightarrow \mathbb{R}^m$ is much richer from a geometric-topological point of view.
When $m=2$, Park--Yun--Lee--Shin \cite{PYLS} showed that
equality holds in (\ref{HS1}) for $n=1$, namely
\begin{equation*} \label{Park}
w_{\text{min}}(1,2) =3.
\end{equation*}
This was extended by Rochau--Chan--Gottschalk \cite{RochauChanGottschalk2024} to 
$$w_{\text{min}}(n,m) \ge \max\{ m,n\}+1, \quad \text{if } m \le 2n,$$
 building on the earlier work of Kim--Min--Park \cite{KimMinPark2024}.
Hwang \cite{Hwang2025} proved that
\begin{equation} \label{Hwanglower}
w_{\text{min}}(2,2) = 4.
\end{equation}

The question of minimal width has also been extensively studied under different hypotheses on the activation function and the approximation space (for example, using  $L^p$ instead of $C^0$ norms). We refer the reader to \cite{Cai2023,Johnson2018, KidgerLyons2020, KimMinPark2024, Lu2017, PYLS, RochauChanGottschalk2024, ShinKimHwangPark2025}
 and the references therein.

The crux of Hwang's argument for (\ref{Hwanglower})  is a two-to-one map from the unit circle $S^1$ to itself, which ``wraps around twice''.
More precisely, one defines a function $f: A \rightarrow \mathbb{R}^2$, for an annulus $A$ in $\mathbb{R}^2$ containing the unit circle $S^1$, with the property that  $f$ maps $S^1$ onto $S^1$ with winding number $2$.   Any neural network approximation to $f$ of width at most $3$ would give rise to an  \emph{injective} map from a simple closed curve in $\mathbb{R}^2$ to the cylinder $S^1 \times \mathbb{R} \subset \mathbb{R}^3$ which, after projection to $S^1$, would wrap around twice. No such injective map exists and hence the width of the network must be at least 4.

Our main result extends Hwang's geometric-topological argument to higher dimensions.  
The map $x \mapsto -x$ sending a point on the sphere $S^{n-1}$ to its antipodal point defines a two-to-one quotient map from $S^{n-1}$ to real projective space $\mathbb{RP}^{n-1}$, the set of lines through the origin in $\mathbb{R}^n$.  Assuming there is a smooth immersion from $\mathbb{RP}^{n-1}$ to $\mathbb{R}^n$, the composition gives a  smooth immersion from $S^{n-1}$ onto its image in $\mathbb{R}^n$.  This map, extended to an annulus in $\mathbb{R}^n$, plays the role of the function $f$ in Hwang's argument.

We prove the following minimal width lower bound for continuous and monotone activations, assuming the existence of smooth immersions of $\mathbb{RP}^{n-1}$ into $\mathbb{R}^n$.  We say that  $\sigma: \mathbb{R} \rightarrow \mathbb{R}$ is \emph{monotone} if it is either non-increasing on $\mathbb{R}$ or non-decreasing on $\mathbb{R}$ (for example, constant functions are monotone).

\begin{teo}\label{width-main} Assume the activation function $\sigma$ is continuous and monotone.
Let $n\geq 2$ and $m\geq 1$. Assume that there is a smooth immersion
$i:\mathbb{RP}^{n-1}\to\mathbb R^n$. Then there exist a compactly supported smooth map
$f_*: \mathbb{R}^n \rightarrow \mathbb R^m$, a compact annulus
$A\subset\mathbb R^n$, and $\varepsilon_0>0$ such that the following
holds. For every
network $F:\mathbb R^n\to\mathbb R^m$ of width at most
$n+\min\{n,m\}-1$, we have
$$
\sup_A |f_*-F|\geq \varepsilon_0.
$$
In particular, $$w_{\mathrm{min}}(n,m)\geq n+\min\{n,m\}.$$
\end{teo}

Combining this bound with (\ref{HS1}) and a well-known
classification of the dimensions $n$ for which $\RP^{n-1}$ admits
a smooth immersion into $\mathbb{R}^n$, we obtain the following.

\begin{corollary}\label{cor-relu}
For the ReLU activation,
$$
\wmin(n,m)=n+m
\quad\text{if }n\in\{2,3,4,7,8\}\text{ and }1\leq m \leq n.
$$
\end{corollary}

The upper bound (\ref{HS1}) extends easily from the ReLU activation to the Leaky ReLU activations defined by
$\sigma(x) =  \max \{ x, ax\},$
for a fixed $a\in (0,1)$.  Therefore, the conclusion of Corollary \ref{cor-relu} also holds for Leaky ReLU activations.

We do not expect that the formula $w_{\text{min}}(n,m) = n+m$ holds for all values of $n, m$ with $1 \le m \le n$, although we cannot yet provide a  counterexample.  On the other hand, it is known that if $m>2n\ge 4$ then $w_{\text{min}}(n,m) = n+m$ fails \cite[Theorem 4.8]{Hwang2025}.  We believe that this question is connected to differential topology in a non-trivial way, and deserves further study.

In  \cite{HwangC1}, Hwang considers the $C^1$
analogue of the minimal width problem studied here, namely approximation of $C^1$ maps in the $W^{1,\infty}_{\mathrm{loc}}$ topology.  He obtains upper bounds for activation functions including ReLU 
and Leaky ReLU, and also a lower bound
for non-decreasing $C^1$ activation functions; this lower bound does not apply to ReLU or
Leaky ReLU.
These bounds are expressed through a geometrically defined quantity $\Omega(n,m)$.
Hwang's arguments for lower bounds on $\Omega(n,m)$, which use submersions of $\mathbb{RP}^{n-1} \times (-1,1)$, are  related to the techniques used here. His arguments rely, however, on stability under $C^1$ perturbations and, as observed in \cite[Section 7]{HwangC1}, therefore do not directly extend to the uniform norm.

The paper is organized as follows. In \S\ref{sec-pre}, we record a needed homotopy theorem. The construction of the map $f$ and its key property Theorem \ref{thm-key} are established in \S\ref{sec-key}, while a more explicit and shorter construction for the $n=2$ case is provided in \S\ref{sectionn2}. Theorem \ref{width-main} is proved in \S\ref{sec-thm-proof}, and Corollary \ref{cor-relu} is proved in \S\ref{sec-cor-proof}. 

\bigskip
\noindent
{\bf AI Disclosure.} \ The authors used ChatGPT to help discover the main results.  All arguments were independently checked by the authors and the exposition is entirely their own.

\section{A persistence theorem for zeros under odd homotopies}\label{sec-pre}

We will make use of the following result on the persistence of zeros under odd homotopies.

\begin{theorem} \label{thmhomotopy}
Let $M$ be a smooth compact $n$-manifold with boundary.  Let $\tau: M \rightarrow M$ be a smooth free involution (i.e. a smooth map with $\tau^2 = \emph{Id}$ and with no points $q$ with $\tau(q)=q$).
Let $\Phi: M \times [0,1] \rightarrow \mathbb{R}^n$ be a smooth map and define $\Phi_t :M \rightarrow \mathbb{R}^n$ by $\Phi_t(x) := \Phi(x,t)$.  Assume the following:
\begin{enumerate}
\item $\Phi_0$ has exactly two zeros on $M$, at $p$ and $\tau(p)$, and the derivative of $\Phi_0$ at $p$  is invertible.  
\item $\Phi_t(\tau (x) ) = - \Phi_t(x)$ for all $t \in [0,1]$. (Namely, $\Phi_t$ is an odd homotopy).
\item $\Phi_t |_{\partial M}$ is nowhere zero, for all $t \in [0,1]$.  
\end{enumerate}
Then $\Phi_1$ has a zero.
\end{theorem}

To prove it, we will require a standard intersection mod 2 result (see \cite{Hirsch1976}, Theorem 2.1 and the discussion on pages 132-133).

\begin{theorem} \label{thmhirsch}
Let $X$ be a smooth $2n$-manifold with boundary and $N\subset X$ be a compact $n$-dimensional submanifold such that $\partial N= N \cap \partial X$ and $N$ meets $\partial X$ transversely.  Let $h_0, h_1: N \rightarrow X$ be smooth maps transverse to $N$, joined by a smooth homotopy $h_t: N \rightarrow X$ for $t \in [0,1]$ such that $h_t(\partial N) \subset \partial X \setminus \partial N.$  Then $h^{-1}_0(N)$ and $h^{-1}_1(N)$ have the same number of points mod $2$.
\end{theorem}

Theorem \ref{thmhomotopy} then follows easily.

\begin{proof}[Proof of Theorem \ref{thmhomotopy}]
Define 
$$X = (M \times \mathbb{R}^n)/ \sim,\quad \text{for } 
(x,y) \sim (\tau (x), -y),$$ and 
$$N = \{ [x,0] \in X \ | \ x\in M \}.$$
Then $X$ is a smooth $2n$-manifold with boundary and $N$ is a compact submanifold of dimension $n$ with boundary 
$\partial N =N \cap \partial X.$ Moreover,
$N$ meets $\partial X$ transversely.  For $t\in[0,1]$, define $h_t: N \rightarrow X$ by
$$h_t([x,0]) = [x, \Phi_t(x)],$$
which is well-defined by (2).  The condition (3) implies that $h_t(\partial N) \subset \partial X \setminus \partial N$.  By (1), the set
 $h_0^{-1}(N)$ is the singleton $\{[p,0]\}$ and the intersection is transverse because $D\Phi_0(p)$ is invertible.  If $\Phi_1$ did not have a zero then $h_1^{-1}(N)$ would be  empty, and $h_1$ would  be transverse to $N$.  But then Theorem \ref{thmhirsch} would imply that, mod 2,
 $$1 = \# h_0^{-1}(N) = \#h_1^{-1}(N)=0,$$
 a contradiction.
\end{proof}

\section{Construction of the map}\label{sec-key}

For a small $\delta>0$, define the compact annulus $A=A(\delta)$ in $\mathbb{R}^n$ by
\[
\begin{split}
A:=  {} & \{ y \in \mathbb{R}^n \ | \ d(y, S^{n-1}) \le \delta\} \\
= {} & \{ (1+s)x \in \mathbb{R}^n \ | \  x\in S^{n-1}, \ -\delta \le s \le\delta \},
\end{split}
\]
where $d$ denotes distance in the Euclidean norm.
We prove the following key result.  (Since the proof is somewhat technical, the reader may wish to first read a simpler proof in the case $n=2$ which is presented in Section \ref{sectionn2}.)

\begin{theorem}\label{thm-key}  Assume there is a smooth immersion $i: \mathbb{RP}^{n-1} \rightarrow \mathbb{R}^n$.  Then there exist $\delta, \ve>0$ and a smooth map $f: A \rightarrow \mathbb{R}^n$ such that the following holds.
  If $g: A \rightarrow \mathbb{R}^n$ and $h: A \rightarrow \mathbb{R}^{n-1}$ are smooth maps, and $g$ satisfies 
$$\sup_{A} |f - g | \le \ve,$$
then $G: A \rightarrow \mathbb{R}^{2n-1}$ given by $G := (g,h)$ is not injective.
\end{theorem}
\begin{proof}
Let
$$\pi: S^{n-1} \rightarrow \mathbb{RP}^{n-1} = S^{n-1}/\sim,$$
be the quotient map $\pi(x)=[x]$, where $x \sim y$ for $x,y \in S^{n-1}$ if $x$ and $y$ are equal or are antipodal points (namely, $x=-y$).   Define
$$b= i \circ \pi : S^{n-1} \rightarrow \mathbb{R}^n,$$
which is again a smooth immersion, since $\pi$ is  a local diffeomorphism.  
Observe that $b(-x)=b(x)$.

Next we will define a smooth map $\nu: S^{n-1} \rightarrow S^{n-1}$ which takes a point $x \in S^{n-1}$ and returns $\nu(x)$, a unit normal vector to $b_*(T_xS^{n-1})$ at the point $b(x)$.  

We first set up some notation.  We define the \emph{orientation} of a basis of vectors $v_1, \ldots, v_n$ of $\mathbb{R}^n$ to be
$$[v_1, \ldots, v_n] = \left\{ \begin{array}{ll} 1, \quad & \det (v_1, \ldots, v_n) >0 \\ -1, \quad & \det (v_1, \ldots, v_n)<0, \end{array} \right.$$
where $(v_1, \ldots, v_n)$ is the $n\times n$ matrix whose column vectors are $v_i$.  Note that replacing one of the $v_i$ by $-v_i$ changes the sign of the orientation.  If $e_1, \ldots, e_n$ are  the standard basis vectors of $\mathbb{R}^n$ then $[e_1, \ldots, e_n]=1$.

Let $v_1, \ldots, v_{n-1}$ be basis vectors of $T_xS^{n-1}$ such that   $[v_1, \ldots, v_{n-1}, x]=1$.  Note that $x$ is normal to $T_xS^{n-1}$ so $v_1, \ldots, v_{n-1}, x$ gives a basis of $\mathbb{R}^n$. 
Define $\nu(x)$  to be the unique unit normal vector to $b_*(T_{x}S^{n-1})$ in $\mathbb{R}^n$ so that $$[Db|_x(v_1), \ldots, Db|_x(v_{n-1}), \nu(x)]=1.$$
Note that since $b$ is an immersion,  $b_*(T_{x}S^{n-1})$ is a hyperplane in $\mathbb{R}^n$ spanned by  the vectors $Db|_x(v_1), \ldots, Db|_x(v_{n-1})$.
 There are exactly two unit normal vectors to $b_*(T_{x}S^{n-1})$; the orientation determines which one we pick.  The definition of $\nu$ is independent of the particular choice of $v_1, \ldots, v_{n-1}$; it only requires that $[v_1, \ldots, v_{n-1}, x]=1$.

We need an elementary lemma.

\begin{lemma} \label{nux}
For all $x \in S^{n-1}$ we have $\nu(-x)=(-1)^n \nu(x)$. 
\end{lemma}
\begin{proof}[Proof of Lemma \ref{nux}]
Fix $x\in S^{n-1}$ and let $v_1,\ldots,v_{n-1}$ be a basis of  $T_x S^{n-1}$ with $[v_1,\ldots,v_{n-1}, x]=1$.  By definition of $\nu(x)$, we have
 $$[Db|_x (v_1),\ldots, Db|_x (v_{n-1}), \nu(x)] =1.$$
On the other hand, $v_1, \ldots, v_{n-2}, -v_{n-1}$ is a basis of $T_{-x}S^{n-1}$ with
$$[v_1, \ldots, v_{n-2}, -v_{n-1}, -x ]=1.$$
  Since $b(-x)=b(x)$ we have $Db|_{-x} = -Db|x$.  Hence
\[
\begin{split}
\lefteqn{ [Db|_{-x}(v_1),\ldots, Db|_{-x}(v_{n-2}), Db|_{-x}(-v_{n-1}), (-1)^n \nu(x)] } \\= {} & (-1)^{n-2+n}   [Db|_x (v_1),\ldots, Db|_x (v_{n-1}), \nu(x)] \\
= {} & 1,
\end{split}
\]
and $\nu(-x) = (-1)^n\nu(x)$ as required.
\end{proof}

Let $\delta>0$ be a small constant to be determined later.    We  now define $f: A \rightarrow \mathbb{R}^n$ by
\begin{equation} \label{f1tx}
f((1+s)x) = b(x) + s \nu(x), \quad x \in S^{n-1}, \ -\delta \le s \le \delta.
\end{equation}
We claim that the derivative $Df|_x$ of $f$ at every $x \in S^{n-1}$ is invertible.  Indeed, without loss of generality, assume that $x=e_n$, the North pole of $S^{n-1}$ and use coordinates $x_1, \ldots, x_{n-1}$ in a neighborhood $U \subset S^{n-1}$ of $x$.  Then $e_1 = \partial_{x_1}, \ldots, e_{n-1} = \partial_{x_{n-1}}$ span $T_xS^{n-1}$.  Compute 
\begin{equation} \label{Df1}
\begin{split}
Df|_x (e_i) = {} & Db|_x (e_i), \quad i=1, \ldots, n-1,  \\
Df|_x (e_n) ={} & \nu(e_n).
\end{split}
\end{equation}
Since $b$ is an immersion, the vectors $Db|_x(e_1), \ldots, Db|_x(e_{n-1})$ span the hyperplane $b_* T_xS^{n-1} \subset \mathbb{R}^n$ and by definition $\nu(e_n)$ is orthogonal to the hyperplane.  Hence $Df|_x$ is invertible, completing the claim.  Assuming that $\delta>0$ is sufficiently small, it follows that $f$ is a local diffeomorphism.

Assume first that $n$ is even.  Then
\begin{equation} \label{feven}
 f(-y) = f(y), \quad \text{for all } y \in A,
 \end{equation}
since $b(-x)=b(x)$ and $\nu(-x)=\nu(x)$.

Define a compact $2n-1$-manifold $M$ by $M = S^{n-1} \times \ov{B}_{\delta}$ where $B_{\delta}$ is the open ball of radius $\delta$ centered at the origin in $\mathbb{R}^n$.  The manifold $M$ has a boundary $S^{n-1} \times \partial B_{\delta}$.    Define a smooth free involution $\tau: M \rightarrow M$ given by 
$$\tau(x, w) = (-x, w).$$  
Observe that if $(x, w) \in M$ then $x+w$ and $-x+w$ lie in the set $A$.

Define a map $\Phi_0: M \rightarrow \mathbb{R}^{2n-1}$ by
$$\Phi_0(x, w) = (   f(x+w) - f(-x+w), x_1, \ldots, x_{n-1}),$$
where $x_1, \ldots, x_{n-1}$ are the first $n-1$ components of $x=(x_1, \ldots, x_n)$.

We first show that $\Phi_0$ satisfies assumption (1) of Theorem \ref{thmhomotopy}.  First note that for fixed $x$ and small $w$, using (\ref{feven}) and Taylor's Theorem,
\begin{equation} \label{ff}
f(x+w)- f(-x+w) =  f(x+w) - f(x-w) = 2Df|_x (w) + O(|w|^2).
\end{equation}
If $\Phi_0(x,w)=0$ then $x_1=\cdots = x_{n-1}=0$ which implies that $x = e_n$ or $-e_n$.  The points $(e_n,0)$ and $(-e_n,0)$ are zeros of $\Phi_0$ in $M$.   Moreover, the derivative of $\Phi_0$ at $(\pm e_n,0)$, using $x_1, \ldots, x_{n-1}$ as local coordinates on $S^{n-1}$ near $\pm e_n$,  is 
\[
(D\Phi_0)|_{(\pm e_n, 0)} = \begin{pmatrix}
2Df|_{\pm e_n}  & 0 \\
0 & I_{n-1} 
\end{pmatrix},
\]
with columns ordered as $w_1, \ldots, w_n, x_1, \ldots, x_{n-1}$.  This is invertible since $Df|_{\pm e_n}$ is invertible.  This implies that there exists an open neighborhood $U$ of $(e_n, 0)$  such that $F_0$ is non-zero on $U \setminus \{ (e_n, 0) \}$, and similarly for $(-e_n, 0)$.  Hence, shrinking $\delta>0$ if necessary, 
 $F_0$ has exactly two zeros $(e_n, 0)$ and $\tau(e_n,0)$ and the derivative of $\Phi_0$ at these zeros is invertible.

Next, define $\Phi_1: M \rightarrow \mathbb{R}^{2n-1}$ by
$$\Phi_1(x, w) = ( g(x+w) - g(-x+w), h(x+w) - h(-x+w))$$
and $\Phi: M \times [0,1] \rightarrow \mathbb{R}^{2n-1}$ by $\Phi(x,t) = \Phi_t(x) := (1-t)\Phi_0(x) + t\Phi_1(x)$.  We have
$$\Phi_t( \tau (x,w)) = \Phi_t( -x, w) = - \Phi_t(x,w),$$
so (2) is satisfied.

Finally, for (3), write $\pi_1:  \mathbb{R}^n \times \mathbb{R}^{n-1} \rightarrow \mathbb{R}^n$ for the projection $\pi_1(x,y)=x$ onto the first factor.
There exists a constant $c>0$ such that 
$$|\pi_1 \circ \Phi_0(x,w)| = |f(x+w)- f(-x+w)| \ge c,$$
for $(x,w) \in \partial M$, since we already showed that $f(x+w)=f(-x+w)$ only when $w=0$, and $\Phi_0$ is continuous on the compact boundary $\partial M$.  Since $f$ and $g$ are uniformly close on $A$ we may choose $\ve>0$ sufficiently small, depending on $\delta$, such that for any $(x,w) \in M$,
$$|\pi_1 \circ \Phi_1(x,w) - \pi_1 \circ \Phi_0(x,w)| = | f(x+w) - f(-x+w) - g(x+w) + g(-x+w)| < c/2.$$
Since
\[
\begin{split}
\pi_1 \circ \Phi_t(x,w) =  {} & (1-t) \pi_1 \circ \Phi_0(x,w) + t \pi_1 \circ \Phi_1(x,w) \\
 =  {} & \pi_1 \circ \Phi_0(x,w) + t(\pi_1 \circ \Phi_1(x,w) - \pi_1 \circ \Phi_0(x,w)),
 \end{split}
 \]
we have for $(x,w) \in \partial M$,
$$|\Phi_t(x,w)| \ge | \pi_1 \circ \Phi_t(x,w)| \ge c/2,$$
and this proves (3).

By Theorem \ref{thmhomotopy}, $\Phi_1$ has a zero, namely a point $(x, w)$ in $M$ such that
$$g(x+w) = g(-x+w), \quad h(x+w)=h(-x+w).$$
Hence $G(x+w)=G(-x+w)$ with $x+w$ and $-x+w$ being distinct points in $A$.  This completes the proof in the case of $n$ even.

Now assume $n$ is odd.  Rather than repeat the entire proof, we merely indicate the necessary changes.  Crucially, we define for $x \in S^{n-1}$ and $w \in \mathbb{R}^n$ the map
$R_x: \mathbb{R}^n \rightarrow \mathbb{R}^n$ by
$$R_x(w) = w - 2 (x\cdot w) x.$$
This satisfies $R_x w=w$ if $w$ is perpendicular to $x$, and $R_x x = -x$, so that the component of $w$ which is in the $x$ direction gets reflected. Observe that $R_{-x}=R_x$ and $R_x^2 =\text{Id}$.
We define the involution $\tau: M \rightarrow M$ by
$$\tau(x,w) = (-x, R_x w),$$
noting that $\tau^2 (x,w) = (x, R_{-x} R_x w)=(x,w)$ as required.

If $(x, w) \in M$ then $x+w$ and $-x+R_x w$ lie in the set $A$.  Define a map $\Phi_0: M \rightarrow \mathbb{R}^{2n-1}$ by
$$\Phi_0(x, w) = ( f(x+w) - f(-x+R_xw), x_1, \ldots, x_{n-1}),$$
where $x_1, \ldots, x_{n-1}$ are the first $n-1$ components of $x=(x_1, \ldots, x_n)$.

Taking the Taylor expansion in $w$ at a fixed $x \in S^{n-1}$ gives
\begin{equation} \label{fmf2}
f(x+w) - f(-x+R_x w) = Df|_x w - (Df|_{-x} \circ R_x)(w)+ O(|w|^2).
\end{equation}
We claim that for all $x \in S^{n-1}$,
\begin{equation} \label{Dfmx}
Df|_{-x} \circ R_x = - Df|_x.
\end{equation}
Indeed, assume without loss of generality that $x=e_n$ and use coordinates $x_1, \ldots, x_{n-1}$ in neighborhoods $U^+, U^- \subset S^{n-1}$ of $x$ and $-x$ respectively.  Then $e_1 = \partial_{x_1}, \ldots, e_{n-1} = \partial_{x_{n-1}}$ span $T_xS^{n-1}$ and $T_{-x}S^{n-1}$.  Recalling that $b(x)=b(-x)$ and $\nu(-x)=-\nu(x)$, we compute from (\ref{f1tx}) and (\ref{Df1}),
\[
\begin{split}
Df|_{-x} (e_i) = {} & Db|_{-x} (e_i)= - Db|_x (e_i) = - Df|_x (e_i), \quad i=1, \ldots, n-1,  \\
Df|_{-x} (e_n) ={} & -\nu(-e_n) = \nu(e_n) = Df|_x (e_n).
\end{split}
\]
On the other hand,  for $x=e_n$, the map $R_x : \mathbb{R}^n \rightarrow \mathbb{R}^n$ is given by  $R_x(e_i)=e_i$  for $i=1, \ldots, n-1$ and $R_x(e_n)=-e_n$. Then
\[ \begin{split}
Df|_{-x} \circ R_x (e_i) = {} & Df|_{-x}(e_i) = - Df|_x(e_i), \quad i=1, \ldots, n-1, \\
Df|_{-x} \circ R_x (e_n) ={} & Df|_{-x} (-e_n) = -Df|_x (e_n),
\end{split} \]
completing the proof of (\ref{Dfmx}).

From (\ref{fmf2}) we have
$$f(x+w) - f(-x+R_x w) = 2Df|_x w + O(|w|^2).$$
Then $\Phi_0$ satisfies (1) by the same proof as in the case of $n$ even.
 
We define $\Phi_t$ in the analogous way using $g$ and $h$.  It is straightforward to check (2) and (3) and so we can again apply Theorem \ref{thmhomotopy} to 
obtain two points $x+w$ and $-x+R_xw$ which are distinct because $w$ is small, with $G(x+w)= G(-x+R_xw)$.
\end{proof}

\section{A simpler proof of Theorem \ref{thm-key} in the case of $n=2$} \label{sectionn2}

When $n=2$ there is a more elementary and explicit proof of Theorem \ref{thm-key}, the case already treated by Hwang \cite{Hwang2025}.  For the reader's convenience, we write down the self-contained proof which is similar in spirit to \cite{Hwang2025}.  Rather than define our function $f$ using an  immersion of $\mathbb{RP}^1=S^1$ into $\mathbb{R}^2$ (of which there are many), we simply write down an explicit function $f$.

Let $f: \mathbb{R}^2 \rightarrow \mathbb{R}^2$ be the map $f(x_1,x_2)=(x_1^2-x_2^2, 2x_1x_2)$, namely the map $z \mapsto z^2$ in complex coordinates.   On the unit circle $\{ (\cos \theta, \sin \theta) \}_{\theta \in [0,2\pi]}$, the map $f$ sends the point with angle $\theta$ to the point with angle $2\theta$,
 and antipodal points $x$ and $-x$  have the same image under $f$.  
 
Define the compact $3$-manifold $M = S^1 \times \ov{B}_{\delta}$, where $B_{\delta}$ is the open ball of radius $\delta$ centered at the origin in $\mathbb{R}^2$.  The boundary of $M$ is $S^1 \times \partial \ov{B}_{\delta}$.  Identify $S^1$ with $\mathbb{R}/2\pi \mathbb{Z}$ via  $$\mathbb{R}/2\pi \mathbb{Z}\ni [\theta] \mapsto x_{\theta}:= (\cos \theta, \sin \theta) \in S^1.$$   There is a smooth free involution of $M$ given by $\tau ([\theta], w) = ([\theta+\pi], w)$.  Observe that if $([\theta], w) \in M$ then $x_{\theta} + w$ and $-x_{\theta}+w$ are in the annulus $A$.
 Define a map
$\Phi_0: M \rightarrow \mathbb{R}^3$ by
$$\Phi_0([\theta], w) = ( f(x_{\theta} + w) - f(-x_{\theta} + w), \cos \theta).$$
More explicitly, writing $w=(w_1,w_2)$, we obtain
\[
\Phi_0([\theta],w)
=
\bigl(
4w_1\cos\theta-4w_2\sin\theta,\;
4w_1\sin\theta+4w_2\cos\theta,\;
\cos \theta
\bigr).
\]
 Observe that $\Phi_0$ has exactly two zeros, at $([\pi/2],0)$ and $([3\pi/2],0)$.  Let's compute the derivative of $\Phi_0$ at $([\pi/2],0)$.  
With columns ordered as $w_1,w_2, \theta$, its derivative is
\[
D\Phi_0|_{([\theta],w)} 
=
\begin{pmatrix}
  4\cos\theta & -4\sin\theta & -4w_1\sin\theta-4w_2\cos\theta\\
 4\sin\theta & 4\cos\theta & 4w_1\cos\theta-4w_2\sin\theta \\
  0 & 0 & -\sin\theta
\end{pmatrix}.
\]
In particular,
\[
D\Phi_0|_{([\pi/2],0)}
=
\begin{pmatrix}
 0 & -4 & 0 \\
4 & 0 & 0\\
  0 & 0 & -1
\end{pmatrix},
\]
which is invertible.  This shows that $\Phi_0$ satisfies assumption (1) of Theorem \ref{thmhomotopy}.

Next, define $\Phi_1: M \rightarrow \mathbb{R}^3$ by
$$\Phi_1([\theta], w) = ( g(x_{\theta} + w) - g(-x_{\theta} + w), h(x_{\theta} + w ) - h(-x_{\theta} + w)),$$
and $\Phi: M \times [0,1] \rightarrow \mathbb{R}^3$ by $\Phi(x,t) = \Phi_t(x):= (1-t) \Phi_0(x) + t \Phi_1(x)$. Observe that
$$\Phi_t(\tau ([\theta], w)) = \Phi_t( [\theta+\pi],w ) =  -\Phi_t([\theta], w).$$
Hence (2) is satisfied.

Finally we check (3).  Let  $([\theta], w)$ be in the boundary $\partial M$.  We have $|w|=\delta$.  Then 
$$|\Phi_0([\theta], w)|^2 \ge  (4w_1 \cos \theta - 4w_2 \sin \theta)^2 + (4w_1 \sin \theta+ 4 w_2 \cos \theta)^2 = 16\delta^2.$$
Since $g$ is close to $f$ on $A$ we may assume, if $\ve>0$ is small enough, that $\Phi_t$ does not vanish on the boundary of $M$ for $t \in [0,1]$.

Then by Theorem \ref{thmhomotopy}, $\Phi_1$ has a zero, namely a point $([\theta], w)$ in $M$ such that
$$g(x_{\theta}+w) = g(-x_{\theta}+w), \quad h(x_{\theta}+w)=h(-x_{\theta}+w).$$
 Hence $G(x_{\theta}+w)=G(-x_{\theta}+w)$ and  $x_{\theta}+w, -x_{\theta}+w$ are distinct points in $A$, completing the proof of Theorem \ref{thm-key} in the case $n=2$.

\section{Proof of Theorem \ref{width-main}}\label{sec-thm-proof}

Besides the key ingredient in \S\ref{sec-key}, we need the following lemma on the networks, which shows that we can approximate neural networks by injective maps composed with projections  (cf. \cite[Section 5]{Johnson2018}).

For integers $N\geq k\geq 1$, write
$$
\pi_{N,k}:\mathbb R^N\to\mathbb R^k,\quad
\pi_{N,k}(x_1,\ldots,x_N):=(x_1,\ldots,x_k).
$$

\begin{lem}\label{width-lift}
Assume the activation function $\sigma$ is continuous and monotone.  Fix  $N\geq\max\{n,m\}$.  Suppose that
$F:\mathbb R^n\to\mathbb R^m$ is a network of width at most
$N$. For every nonempty compact $K\subset\mathbb R^n$ and every
$\varepsilon>0$, there is a smooth injective map
$G:\mathbb R^n\to\mathbb R^N$ such that
$$
\sup_K |\pi_{N,m}\circ G-F| < \varepsilon.
$$
\end{lem}

\begin{proof}
We first approximate $\sigma$ by smooth strictly monotone functions.
Choose $\eta\in\{1,-1\}$ such that $\eta\sigma$ is non-decreasing.
Choose a standard nonnegative mollifier
$\varphi\in C_c^\infty(-1,1)$ with
$\int_{\mathbb R}\varphi(r)\,dr=1$. For $t\ge 0$, define
$$
\sigma_t(s)
:=\int_{-1}^1\sigma(s-tr)\varphi(r)\,dr+\eta ts,\quad s\in \mathbb{R}. 
$$
Clearly, $\sigma_0=\sigma$.
By standard properties of mollification, $\sigma_t$ is smooth
for every $t>0$ and
$\sigma_t\to\sigma$ locally uniformly as $t\to 0^+$.
For $t>0$,
we verify that $\sigma_t$ is strictly monotone and thus injective on $\mathbb{R}$. Indeed,
for $s_1<s_2$, monotonicity of $\sigma$ and $\varphi\geq 0$ give
\begin{align*}
\eta\sigma_t(s_2)-\eta\sigma_t(s_1)
&=\int_{-1}^1
\bigl(\eta\sigma(s_2-tr)-\eta\sigma(s_1-tr)\bigr)
\varphi(r)\,dr+t(s_2-s_1)\\
&\geq t(s_2-s_1)>0.
\end{align*}

We again use $\sigma_t$ to denote its
higher dimensional extensions in the same way as for $\sigma$. Then $\sigma_t$ as a map from $\mathbb R^N$ to $\mathbb{R}^N$ is smooth and
injective for every $t>0$.

Next we put all the affine maps of the network on the same space
$\mathbb R^N$. Let $F$ be expressed by affine maps
$A^{(1)},\ldots,A^{(L)}$ and dimensions
$d_0=n$, $d_L=m$, with $d_{\ell} \leq N$ for $\ell=0, \ldots, L$. For $t\ge 0$, define $\overline A_t^{(\ell)}: \mathbb{R}^N \rightarrow \mathbb{R}^N$ by
$$
\overline A_t^{(\ell)}(x)
:=\bigl(A^{(\ell)}(x_1,\ldots,x_{d_{\ell-1}}),
{0,\ldots,0}\bigr) + tx, 
\quad 1\leq \ell\leq L,
$$
for $x=(x_1,\dots,x_N)$.
Here there are $(N-d_{\ell})$ zeros in the above definition;
when $d_{\ell}=N$, there are no added zero coordinates. There is a finite set $E\subset
\mathbb R$ such that for $t\notin E$, the affine map
$\overline A_t^{(\ell)}$ is an isomorphism.  This follows from the fact that if $M$ is any $N\times N$ matrix and $I_N$  the $N\times N$ identity matrix, $\det(M+tI_N)$ is a polynomial in $t$ of degree $N$ and hence vanishes for at most $N$ values of $t$.



Let
$J:\mathbb R^n\to\mathbb R^N$ be defined by $J(x):=(x,0,\ldots,0)$, 
and define
$$
G_t:=\overline A_t^{(L)}\circ\sigma_t\circ
\overline A_t^{(L-1)}\circ\cdots\circ\sigma_t\circ
\overline A_t^{(1)}\circ J.
$$
When $L=1$, this means $G_t=\overline A_t^{(1)}\circ J$.
For every $t>0$ with $t\notin E$, all factors are smooth and
injective, including $J$. Their composition $G_t$ is consequently
smooth and injective on $\mathbb R^n$.
Moreover, 
$$
G_0(x)=(F(x),0,\ldots,0).
$$

Since $(t,x)\mapsto G_t(x)$ is continuous on the compact set
$[0,1]\times K$, we obtain, for all sufficiently small $t>0$,
$$
\sup_K |G_t-G_0|<\varepsilon.
$$
Choose such a positive $t$ outside the finite set $E$, and let
$G:=G_t$. Then, by $\pi_{N,m}\circ G_0=F$, we have
$$
\sup_K |\pi_{N,m}\circ G-F|
\leq\sup_K |G_t-G_0|<\varepsilon,
$$
which proves the lemma.
\end{proof}

Now we are ready to present the following:

\begin{proof}[Proof of Theorem~\ref{width-main}]
Let $\varepsilon, \delta$, $A$, $f$ be those given by Theorem \ref{thm-key}.
We may assume $\delta<1/4$.
Note that the same defining formula \eqref{f1tx} extends $f$ to a smooth map on $\mathbb{R}^n\setminus \{0\}$. 
Choose a standard cutoff $\chi\in C_c^\infty(\mathbb R)$ supported on $(1/2,3/2)$ such that $0\leq\chi\leq 1$ on $\mathbb R$ and
$\chi=1$ on $[3/4,5/4]$. Define
$$
\overline f(y):=\chi(|y|) f(y) \quad \text{for }y\in \mathbb{R}^n\setminus\{0\},\quad \text{and}\quad 
\overline{f}(0)=0.
$$
Then $\overline f\in C_c^\infty(\mathbb R^n;\mathbb R^n)$ agrees with $f$ on an open neighborhood of $A$.

Let
$$
 \varepsilon_0:=\varepsilon/2,\quad 
d:=\min\{n,m\},\quad \text{and}\quad  N:=n+d-1.
$$
Write $f=(f_1,\ldots,f_n)$ on $A$ and
$\overline f=(\overline f_1,\ldots,\overline f_n)$ on
$\mathbb R^n$.
Define
$$
f_*:=(\overline f_1,\ldots,\overline f_d,
{0,\ldots,0})
\in C_c^\infty(\mathbb R^n;\mathbb R^m).
$$
Here there are $(m-d)$ zeros in the definition of $f_*$.

Fix $\sigma$ to be any continuous and monotone activation. Suppose the contrary that a network $F:\mathbb R^n\to
\mathbb R^m$ of width at most $N$ satisfies
$$
\sup_A |f_*-F|<\varepsilon_0.
$$
By the definition of the width, we must have $N\geq\max\{n,m\}$.  We may therefore apply Lemma~\ref{width-lift}, with $K=A$,
to obtain a smooth injective map $G:\mathbb R^n\to\mathbb R^N$ such
that
$$
\sup_A |\pi_{N,m}\circ G-F|<\varepsilon_0.
$$

Write
$G=(g,h)$,
where $g:\mathbb R^n\to\mathbb R^d$ consists of the first $d$ coordinates and $h:\mathbb R^n\to\mathbb R^{n-1}$ consists of
the remaining $N-d=n-1$ coordinates. By $d\leq m$, the above inequalities yield
\[
\begin{split}
\sup_A |g-(f_1,\ldots,f_d)|
\le {} & \sup_A |\pi_{N,m} \circ G - f_*| \\
\le {}  & \sup_A |\pi_{N,m}\circ G-F|
       +\sup_A |f_* - F| \\
       < {} & 2\varepsilon_0=\varepsilon.
\end{split}
\] 
Here we used $\overline f_j=f_j$ on $A$.

If $d<n$, define
$\widehat g:=(g_1,\ldots,g_d,f_{d+1},\ldots,f_n):A\to\mathbb R^n$.
If $d=n$, simply let $\widehat g:=g|_A$. In either case,
$\widehat g$ is smooth on a neighborhood of $A$ and
$$
\sup_A|\widehat g-f|
=\sup_A |g-(f_1,\ldots,f_d)|<\varepsilon.
$$
Then Theorem \ref{thm-key} shows that the map
$(\widehat g,h):A\to\mathbb R^{2n-1}$
is not injective. Hence there are distinct points $y,z\in A$ with
$\widehat g(y)=\widehat g(z)$ and $h(y)=h(z)$.
The first $d$ coordinates of the first equation give
$g(y)=g(z)$. Together with the second equation, this implies
$$
G(y)=(g(y),h(y))=(g(z),h(z))=G(z),
$$
contradicting the injectivity of $G$.
The theorem is therefore proved.
\end{proof}

\section{Proof of Corollary \ref{cor-relu}}\label{sec-cor-proof}
The following theorem is well-known in differential topology.
\begin{theorem}\label{thm:projective-codimension-one}
For $n\ge 2$, the real projective space
$\mathbb{RP}^{n-1}$ admits a smooth immersion into $\mathbb{R}^n$
if and only if
$n \in \{ 2,3,4,7,8\}$.
\end{theorem}

We recall how this classification follows from the classical results
recorded in \cite{Hirsch1959}. The case $n=2$ is immediate, since
$\mathbb{RP}^{1}$ is diffeomorphic to the circle.

Suppose first that $n\geq 3$ is odd.
By \cite[Theorem~7.2]{Hirsch1959},
$\mathbb{RP}^{n-1}$ admits a smooth immersion into $\mathbb{R}^{n}$
if and only if $\mathbb{RP}^{n}$ is parallelizable (i.e. admits a global smooth frame).
According to results of Kervaire, Bott, and Milnor recalled in
\cite[footnote 3,
p. 272]{Hirsch1959}, $\RP^n$ is parallelizable precisely
when $n$ is $3$ or $7$.

Now suppose that $n\geq 4$ is even, and let $i:\mathbb{RP}^{n-1}\to \mathbb{R}^n$ be a smooth immersion. Write $\pi:S^{n-1}\to \RP^{n-1}$ as the antipodal quotient map.
As in the proof of \cite[Theorem~7.2]{Hirsch1959}, 
$i\circ \pi$ gives an immersion
$S^{n-1}\to\mathbb{R}^n$ of even normal degree, then
Milnor's result implies that $S^{n-1}$ is parallelizable.
Then the result recalled in
\cite[footnote 3, p. 272]{Hirsch1959} forces
$n=4$ or $n=8$.
Conversely, the required immersions for $n=4,8$ are recorded
in \cite[Theorem~7.1]{Hirsch1959}.
Combining these cases gives precisely $n\in\{2,3,4,7,8\}$.

\begin{proof}[Proof of Corollary \ref{cor-relu}]
The desired conclusion follows directly from Theorem \ref{width-main}, Theorem \ref{thm:projective-codimension-one}, and the upper bound \eqref{HS1}.
\end{proof}

\end{document}